\documentclass{article}
\usepackage{amssymb}

\usepackage{amsthm}
\expandafter\ifx\csname pdfoptionalwaysusepdfpagebox\endcsname\relax\else
\fi

\usepackage{graphics} 
\usepackage{epsfig} 

\usepackage{graphicx}
\usepackage{linegoal}
\makeatletter
\def\endfigure{\end@float}
\def\endtable{\end@float}
\makeatother

\usepackage[caption=false]{subfig}
\usepackage{caption}

\usepackage{amsmath, amsfonts, amssymb}
\usepackage{footnote}
\usepackage[hang,flushmargin]{footmisc}

\usepackage{hyperref}

\usepackage[utf8]{inputenc}
\usepackage[T1]{fontenc}
\usepackage{subfiles}

\usepackage{microtype}
\usepackage{parskip}
\usepackage{ragged2e}
\usepackage[x11names,dvipsnames,table]{xcolor}
\usepackage{soul}
\usepackage{pdfpages}

\usepackage[shortlabels]{enumitem}

\usepackage{dsfont}
\usepackage{mdframed}
\newcommand{\newchangesnorev}[1]{#1}

\makeatletter
\newcommand{\removelatexerror}{\let\@latex@error\@gobble}
\makeatother

\newcommand{\lwdth}{2pt}
\mdfsetup {
        linewidth = \lwdth,
        innerleftmargin = 0pt,
        innerrightmargin = 7pt,
        innertopmargin = 0,
        innerbottommargin = 0,
         linecolor=Black,
        }

\usepackage{siunitx}
\usepackage{multirow}
\usepackage{rotating}
\usepackage{makecell}
\usepackage{booktabs}
\makeatletter
\def\thickhline{%
  \noalign{\ifnum0=`}\fi\hrule \@height \thickarrayrulewidth \futurelet
   \reserved@a\@xthickhline}
\def\@xthickhline{\ifx\reserved@a\thickhline
               \vskip\doublerulesep
               \vskip-\thickarrayrulewidth
             \fi
      \ifnum0=`{\fi}}
\makeatother

\newlength{\thickarrayrulewidth}
\usepackage{makecell} %
\usepackage{array}
\newcommand{\PreserveBackslash}[1]{\let\temp=\\#1\let\\=\temp}
\newcolumntype{L}[1]{>{\raggedright\let\newline\\\arraybackslash\hspace{0pt}}m{#1}}
\newcolumntype{C}[1]{>{\centering\let\newline\\\arraybackslash\hspace{0pt}}m{#1}}
\newcolumntype{R}[1]{>{\raggedleft\let\newline\\\arraybackslash\hspace{0pt}}m{#1}}
\usepackage{supertabular}

\usepackage[ruled,norelsize,vlined]{algorithm2e}
\usepackage{algorithmicx}

\usepackage{scalerel}
\usepackage[capitalize]{cleveref} 
\crefformat{equation}{(#2#1#3)} 
\usepackage{amsfonts, mathtools}
\usepackage{bm}
\usepackage{cuted}
\usepackage{etoolbox}
\usepackage{upgreek}
\usepackage{nicefrac} 

\makeatletter

\def\@endtheorem{\endtrivlist}
\makeatother

\newtheorem{problem}{Problem}

\newtheorem{definition}{Definition}
\newtheorem{corollary}{Corollary}
\newtheorem{example}{Example}
\newtheorem{remark}{Remark}
\newtheorem{theorem}{Theorem}
\newtheoremstyle{proofstyle}
  {0em}    
  {0em}    
  {\itshape}   
  {}           
  {\bfseries}  
  {.}          
  {.2em}       
  {}          

\theoremstyle{proofstyle}

\newcommand{\mc}[1]{\mathcal{#1}}

\newcommand{\ith}{i^{\text{th}}}

\newcommand{\jth}{j^{\text{th}}}

\newcommand{\x}{\mathbf{x}}

\usepackage{upgreek}

\usepackage[mode=buildnew]{standalone}
\usepackage{adjustbox}

\usepackage{nicematrix}

\crefformat{assumption}{Assumption #1}

\crefformat{equation}{(#2#1#3)} 

\crefformat{assumption}{Assumption #1}

\makeatletter
\def\@opargbegintheorem#1#2#3{\trivlist
   \item[]{\itshape #1\ #2\ (#3)}\\}
\makeatother

\newcommand{\Prob}[2][]{%
  \operatorname{Pr}\ifthenelse{\isempty{#1}}{}{\left(#2\right)}%
}

\newcommand{\cgoal}{{c_{g}}}

\newlist{prbdes}{itemize}{1}
\setlist[prbdes]{label={${\mathcal{P}_1}$:\hspace*{-12pt}}}
\makeatletter
\patchcmd{\enit@itemize@i}{\fi}{\fi\ifnum\pdfstrcmp{\@currenvir}{prbdes}=0 \item\fi}{}{}
\makeatother

\usepackage{mathrsfs}

\usepackage{silence}
\usepackage{tikz}
\usetikzlibrary{spy}

\newcolumntype{f}{>{\centering\arraybackslash}p{.2\columnwidth}}
\newcolumntype{s}{>{\centering\arraybackslash}p{.08\columnwidth}}
\newcolumntype{z}{>{\centering\arraybackslash}p{.1\columnwidth}}

\usepackage{tabstackengine}
\setstackEOL{\cr}

\usepackage{anyfontsize}

\usepackage{siunitx}

\setlist[prbdes]{label={${\mathcal{P}_1}$:\hspace*{-12pt}}}
\makeatletter
\patchcmd{\enit@itemize@i}{\fi}{\fi\ifnum\pdfstrcmp{\@currenvir}{prbdes}=0 \item\fi}{}{}
\makeatother

\newcommand{\numquantiles}{N_{\tau}}

\definecolor{aquaenv}{HTML}{03FFFE}
\definecolor{purpenv}{HTML}{8783EA}
\definecolor{greenenv}{HTML}{7BE76D}
\definecolor{redenv}{HTML}{FF0F42}
\definecolor{mplblue}{RGB}{31, 119, 180}
\definecolor{arxblue}{RGB}{48, 51, 154}
\definecolor{mplred}{RGB}{214, 39, 40}
\definecolor{pargreen}{HTML}{217222}
\definecolor{parblue}{HTML}{165887}
\definecolor{parorange}{HTML}{FABE86}

\definecolor{mplgrey}{RGB}{127, 127, 127}
\definecolor{darkgray}{rgb}{0.663,0.663,0.663}

\newcommand{\tikzcircle}[2][red,fill=red]{\tikz[baseline=-0.5ex]\draw[#1,radius=#2] (0,0) circle ;}%
\newrobustcmd*{\tikzsquare}[1]{\tikz{\filldraw[draw=black,fill=#1] (0,0)
rectangle (0.2cm,0.2cm);}}

\newrobustcmd*{\mycircle}[1]{\tikz{\filldraw[draw=#1,fill=#1] (0,0) circle [radius=0.1cm];}}

\newrobustcmd*{\tikztriangle}[1]{\tikz{\filldraw[draw=black,fill=#1] (0,0) --
(0.2cm,0) -- (0.1cm,0.2cm);}}

\colorlet{DarkAquaEnv}{Aquamarine}
\colorlet{DarkPurpEnv}{RoyalPurple}
\colorlet{DarkGreenEnv}{ForestGreen}
\colorlet{DarkRedEnv}{redenv}

\newcommand{\avi}{AVI}
\newcommand{\qravi}{QR-AVI}
\newcommand{\eqravi}{\mathbb{E}\text{-\qravi}}
\newcommand{\rsqravi}{\rho\text{-\qravi}}

\usepackage{tikz}
\usepackage{pgfplots}
\pgfplotsset{compat=1.18}
\usetikzlibrary{spy,calc}
\newif\ifblackandwhitecycle
\gdef\patternnumber{0}

\pgfkeys{/tikz/.cd,
    zoombox paths/.style={
        draw=orange,
        very thick
    },
    black and white/.is choice,
    black and white/.default=static,
    black and white/static/.style={ 
        draw=white,   
        zoombox paths/.append style={
            draw=white,
            postaction={
                draw=black,
                loosely dashed
            }
        }
    },
    black and white/static/.code={
        \gdef\patternnumber{1}
    },
    black and white/cycle/.code={
        \blackandwhitecycletrue
        \gdef\patternnumber{1}
    },
    black and white pattern/.is choice,
    black and white pattern/0/.style={},
    black and white pattern/1/.style={    
            draw=white,
            postaction={
                draw=black,
                dash pattern=on 2pt off 2pt
            }
    },
    black and white pattern/2/.style={    
            draw=white,
            postaction={
                draw=black,
                dash pattern=on 4pt off 4pt
            }
    },
    black and white pattern/3/.style={    
            draw=white,
            postaction={
                draw=black,
                dash pattern=on 4pt off 4pt on 1pt off 4pt
            }
    },
    black and white pattern/4/.style={    
            draw=white,
            postaction={
                draw=black,
                dash pattern=on 4pt off 2pt on 2 pt off 2pt on 2 pt off 2pt
            }
    },
    zoomboxarray inner gap/.initial=5pt,
    zoomboxarray columns/.initial=2,
    zoomboxarray rows/.initial=2,
    subfigurename/.initial={},
    figurename/.initial={zoombox},
    zoomboxarray/.style={
        execute at begin picture={
            \begin{scope}[
                spy using outlines={%
                    zoombox paths,
                    width=\imagewidth / \pgfkeysvalueof{/tikz/zoomboxarray columns} - (\pgfkeysvalueof{/tikz/zoomboxarray columns} - 1) / \pgfkeysvalueof{/tikz/zoomboxarray columns} * \pgfkeysvalueof{/tikz/zoomboxarray inner gap} -\pgflinewidth,
                    height=\imageheight / \pgfkeysvalueof{/tikz/zoomboxarray rows} - (\pgfkeysvalueof{/tikz/zoomboxarray rows} - 1) / \pgfkeysvalueof{/tikz/zoomboxarray rows} * \pgfkeysvalueof{/tikz/zoomboxarray inner gap}-\pgflinewidth,
                    magnification=3,
                    every spy on node/.style={
                        zoombox paths
                    },
                    every spy in node/.style={
                        zoombox paths
                    }
                }
            ]
        },
        execute at end picture={
            \end{scope}
            \node at (image.north) [anchor=north,inner sep=0pt] {\subfloat[\label{\pgfkeysvalueof{/tikz/figurename}-image}]{\phantomimage}};
            \node at (zoomboxes container.north) [anchor=north,inner sep=0pt] {\subfloat[\label{\pgfkeysvalueof{/tikz/figurename}-zoom}]{\phantomimage}};
     \gdef\patternnumber{0}
        },
        spymargin/.initial=0.5em,
        zoomboxes xshift/.initial=1,
        zoomboxes right/.code=\pgfkeys{/tikz/zoomboxes xshift=1},
        zoomboxes left/.code=\pgfkeys{/tikz/zoomboxes xshift=-1},
        zoomboxes yshift/.initial=0,
        zoomboxes above/.code={
            \pgfkeys{/tikz/zoomboxes yshift=1},
            \pgfkeys{/tikz/zoomboxes xshift=0}
        },
        zoomboxes below/.code={
            \pgfkeys{/tikz/zoomboxes yshift=-1},
            \pgfkeys{/tikz/zoomboxes xshift=0}
        },
        caption margin/.initial=4ex,
    },
    adjust caption spacing/.code={},
    image container/.style={
        inner sep=0pt,
        at=(image.north),
        anchor=north,
        adjust caption spacing
    },
    zoomboxes container/.style={
        inner sep=0pt,
        at=(image.north),
        anchor=north,
        name=zoomboxes container,
        xshift=\pgfkeysvalueof{/tikz/zoomboxes xshift}*(\imagewidth+\pgfkeysvalueof{/tikz/spymargin}),
        yshift=\pgfkeysvalueof{/tikz/zoomboxes yshift}*(\imageheight+\pgfkeysvalueof{/tikz/spymargin}+\pgfkeysvalueof{/tikz/caption margin}),
        adjust caption spacing
    },
    calculate dimensions/.code={
        \pgfpointdiff{\pgfpointanchor{image}{south west} }{\pgfpointanchor{image}{north east} }
        \pgfgetlastxy{\imagewidth}{\imageheight}
        \global\let\imagewidth=\imagewidth
        \global\let\imageheight=\imageheight
        \gdef\columncount{1}
        \gdef\rowcount{1}
        
    },
    image node/.style={
        inner sep=0pt,
        name=image,
        anchor=south west,
        append after command={
            [calculate dimensions]
            node [image container,subfigurename=\pgfkeysvalueof{/tikz/figurename}-image] {\phantomimage}
            node [zoomboxes container,subfigurename=\pgfkeysvalueof{/tikz/figurename}-zoom] {\phantomimage}
        }
    },
    color code/.style={
        zoombox paths/.append style={draw=#1}
    },
    connect zoomboxes/.style={
    spy connection path={\draw[draw=none,zoombox paths] (tikzspyonnode) -- (tikzspyinnode);}
    },
    help grid code/.code={
        \begin{scope}[
                x={(image.south east)},
                y={(image.north west)},
                font=\footnotesize,
                help lines,
                overlay
            ]
            \foreach \x in {0,1,...,9} { 
                \draw(\x/10,0) -- (\x/10,1);
                \node [anchor=north] at (\x/10,0) {0.\x};
            }
            \foreach \y in {0,1,...,9} {
                \draw(0,\y/10) -- (1,\y/10);                        \node [anchor=east] at (0,\y/10) {0.\y};
            }
        \end{scope}    
    },
    help grid/.style={
        append after command={
            [help grid code]
        }
    },
}

\newcommand\phantomimage{%
    \phantom{%
        \rule{\imagewidth}{\imageheight}%
    }%
}

\newcounter{zoomboxnumber}
\newcommand\zoombox[3][]{
    \begin{scope}[zoombox paths]
        \stepcounter{zoomboxnumber}  
        \pgfmathsetmacro\xpos{
            (\columncount-1)*(\imagewidth / \pgfkeysvalueof{/tikz/zoomboxarray columns} + \pgfkeysvalueof{/tikz/zoomboxarray inner gap} / \pgfkeysvalueof{/tikz/zoomboxarray columns} ) + \pgflinewidth
        }
        \pgfmathsetmacro\ypos{
            (\rowcount-1)*( \imageheight / \pgfkeysvalueof{/tikz/zoomboxarray rows} + \pgfkeysvalueof{/tikz/zoomboxarray inner gap} / \pgfkeysvalueof{/tikz/zoomboxarray rows} ) + 0.5*\pgflinewidth
        }
        \edef\dospy{\noexpand\spy [
            #1,
            zoombox paths/.append style={
                black and white pattern=\patternnumber
            },
            every spy on node/.append style={#1},
            x=\imagewidth,
            y=\imageheight
        ] on (#2) in node [anchor=north west] (\thezoomboxnumber) at ($(zoomboxes container.north west)+(\xpos pt,-\ypos pt)$) {}
        node[] at ($(\thezoomboxnumber.north)+(9mm,-2mm)$) {#3};
        }
        \dospy
        \pgfmathtruncatemacro\pgfmathresult{ifthenelse(\columncount==\pgfkeysvalueof{/tikz/zoomboxarray columns},\rowcount+1,\rowcount)}
        \global\let\rowcount=\pgfmathresult
        \pgfmathtruncatemacro\pgfmathresult{ifthenelse(\columncount==\pgfkeysvalueof{/tikz/zoomboxarray columns},1,\columncount+1)}
        \global\let\columncount=\pgfmathresult
        \ifblackandwhitecycle
            \pgfmathtruncatemacro{\newpatternnumber}{\patternnumber+1}
            \global\edef\patternnumber{\newpatternnumber}
        \fi
    \end{scope}
}

\usepackage{tcolorbox}

\SetCommentSty{mycommfont}

\usepackage[preprint]{corl_2025} 

\title{Safety-Aware Reinforcement Learning for Control via Risk-Sensitive Action-Value Iteration and Quantile Regression}

\author{
  Clinton Enwerem, Aniruddh G.~Puranic, John S.~Baras, and Calin Belta\\
  Institute for Systems Research\\
  University of Maryland, College Park, United States\\
  \texttt{\{enwerem, puranic, baras, calin\}@umd.edu} \\
}

\begin{document}
\maketitle


\begin{abstract}
Mainstream approximate action-value iteration reinforcement learning (RL) algorithms suffer from overestimation bias, leading to suboptimal policies in high-variance stochastic environments. Quantile-based action-value iteration methods reduce this bias by learning a distribution of the expected cost-to-go using quantile regression. However, ensuring that the learned policy satisfies safety constraints remains a challenge when these constraints are not explicitly integrated into the RL framework. Existing methods often require complex neural architectures or manual tradeoffs due to combined cost functions. To address this, we propose a risk-regularized quantile-based algorithm integrating Conditional Value-at-Risk (CVaR) to enforce safety without complex architectures. We also provide theoretical guarantees on the contraction properties of the risk-sensitive distributional Bellman operator in Wasserstein space, ensuring convergence to a unique cost distribution. Simulations of a mobile robot in a dynamic reach-avoid task show that our approach leads to more goal successes, fewer collisions, and better safety-performance trade-offs than risk-neutral methods.
\end{abstract}

\section{Introduction}
\label{sec:intro}

Standard Reinforcement Learning (RL) methods optimize expected cumulative cost but often use \textit{misaligned} or \textit{underspecified} cost functions that oversimplify real-world safety constraints, leading to overlooked risks. This limitation hinders deployment in high-risk domains where safety is crucial \cite{pan2022effects}. For instance, a risk-neutral autonomous robot may reach deadlock states, jeopardizing its mission and causing severe consequences \cite{beyer2014risk}. Addressing this requires integrating explicit safety constraints to prevent risky behavior while preserving essential exploration.
\begin{figure*}[htb]
\centering
\setlength{\fboxsep}{0pt}
\subfloat[Task space.]{%
    \label{fig:simenv}%
    \fbox{\includegraphics[width=.22\columnwidth]{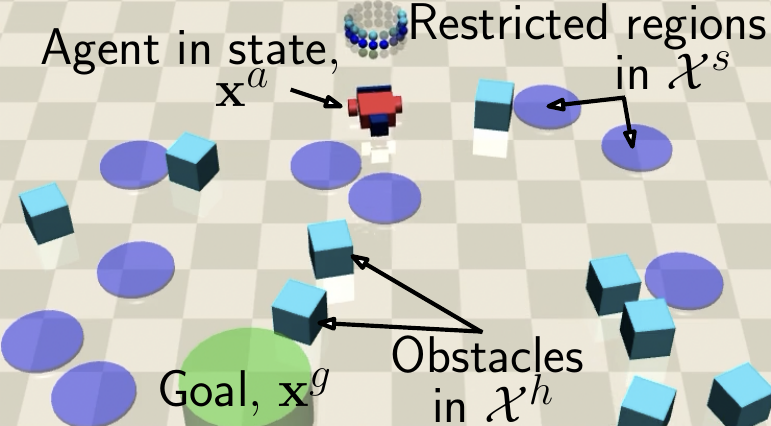}}} 
\hspace{0.05cm}
\subfloat[Unsafe traversal.]{%
    \label{fig:softconst}%
    \fbox{\includegraphics[width=.22\columnwidth, trim={0 0 0cm 0cm},clip]{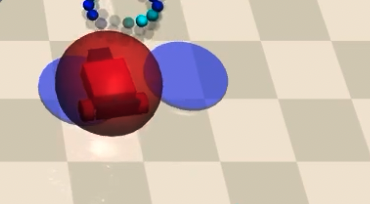}}} 
\hspace{0.05cm}
\subfloat[Collision event.]{%
    \label{fig:hardcon}%
 \fbox{\includegraphics[width=.22\columnwidth, trim={0 0 0cm 0cm},clip]{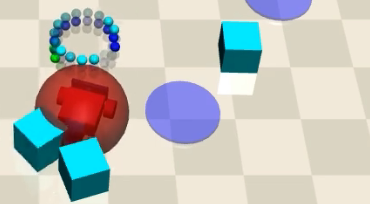}}}
\hspace{0.05cm}
\subfloat[Goal reached.]{%
    \fbox{\includegraphics[width=.22\columnwidth]{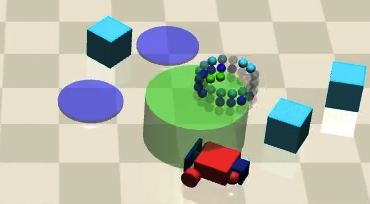}}}
\caption{%
{\textbf{Reach-avoid navigation task}: Close-ups from our experiments (\cref{sec:exp}) showing a differentially-driven mobile robot (depicted by a red car-like object) tasked with navigating to a uniformly randomized 2D goal location (green cylinder), denoted by \(\mathbf{x}^g = [x^g, y^g]^\top\), with \(x^g, y^g \in [-1, 1]\). En route to the goal, the robot must also avoid two regions within its environment: a traversable region, denoted by \(\mathcal{X}^s\) comprising purple discs shown in (b) (acting as a \emph{soft} constraint), as well as obstacles \newchangesnorev{in the set \(\mathcal{X}^h\)} (represented by light blue cubes in (c)) that inhibit its motion upon collision (i.e., a hard constraint). These static hazards and obstacles incur \emph{distinct} costs that the agent must minimize while learning to safely navigate to the goal.}}
\label{fig:learn-env}
\end{figure*}

Approximate Action-Value Iteration (hereafter, {\avi}) methods \cite{mnihPlayingAtariDeep2013,van2016deep} estimate the action-value function using a neural network trained on the temporal difference between a current estimate of the action value and an observed \textit{target} value (i.e., the sum of the stage cost and the expected future action value). However, in high-variance environments, these methods often suffer from overestimation bias, where the action-value approximation predicts values that are significantly higher than the true values \cite{kuznetsovControllingOverestimationBias2020}. {\avi} methods based on quantile regression \cite{kuznetsovControllingOverestimationBias2020,dabneyDistributionalReinforcementLearning2017} address this by learning a quantile distribution of Q-values instead of a single expected value, reducing overestimation and improving estimation accuracy. By updating value distributions across quantile fractions, quantile regression RL algorithms enhance exploration and stabilize policy learning \cite{kuznetsovControllingOverestimationBias2020}.

While distributional RL mitigates overestimation bias, it struggles with instantaneous stage cost misspecification, leading to unsafe or suboptimal policies \cite{pan2022effects}. Modifying stage costs to encode safety constraints \cite{tasseROSARLRewardOnlySafe2023,thomasSafeReinforcementLearning2021} is common but challenging to tune, especially in environments with varied risks \cite{farazi2021deep} (\cref{fig:learn-env}). Safe Policy Optimization \cite{ji2023safety} tackles this problem by enforcing safety through a separate network, which adds substantial computational overhead \cite{fakoorP3OPolicyonPolicyoff2020}. Furthermore, in dynamic environments, uncertainty from stochastic transitions, noise, perturbations, and model approximations can lead to brittle policies if not properly addressed \cite{yingSafeReinforcementLearning2022}. To balance instantaneous stage cost optimization, safety adherence, and uncertainty mitigation -- without the overhead of a separate cost network \cite{ji2023safety,fakoorP3OPolicyonPolicyoff2020} -- we propose an RL framework that integrates risk-sensitive decision-making into approximate action-value iteration algorithms. Specifically, our approach penalizes safety violations by augmenting the quantile loss with a risk term derived from experience-based cost distributions.

{\normalfont\textbf{Contributions and Paper Outline}}
\begin{enumerate}[label=\roman*.,leftmargin=*]
    \item \textit{Risk-regularized Quantile-Regression Based Approximate Action-Value Iteration ({\qravi})} (\cref{sec:approach}): We enhance efficiency and safety by training a single quantile-based action-value iteration using a risk-regularized loss that integrates quantile regression with a penalty for safety violations (\cref{ssec:riskm}).
    \item \textit{Empirical uncertainty quantification via Kernel Density Estimation} (\cref{ssec:uncqnt,ssec:kdecostdist}):
    We approximate the cost distribution using Kernel Density Estimation (KDE), providing  a probabilistic measure of safety constraints based on cost samples and enabling risk computation.
    \item \textit{Contraction \& Finite-Time Convergence Guarantees} (\cref{sec:thprop}): We prove the contraction of the risk-regularized Bellman operator, ensuring stable learning and finite-time convergence of the quantile-based action-value iteration algorithm to a stationary policy, leveraging the convexity of the quantile and risk loss terms.
    \item \textit{Modular RL Algorithmic Framework} (\cref{sec:approach}): We introduce a risk-sensitive QR-AVI algorithm (\cref{alg:raqravi}) within a flexible and modular framework that integrates safety constraints into the quantile loss, making our approach applicable to other approximate RL methods.
    \item \textit{Reach-Avoid Experiments (\cref{{sec:exp}}) \&  Comparative Study (\cref{ssec:benchm})}: We evaluate a car-like agent in a dynamic reach-avoid task, comparing our algorithm to nominal and risk-neutral variants. Results show improved safety, lower quantile loss, and higher success rates.
\end{enumerate}

\section{Background}
\label{sec:prelim}

\subsection{Reinforcement Learning (RL)}
\label{sec:RL}
The terminology in this subsection follow \cite{bertsekas_model_2024}. We consider an {\em agent} as a decision-making entity that learns from interacting with an external {\em environment} through a sequence of {\em observations} (of its states and those of the environment), {\em actions}, and {\em costs}. At each discrete time step \(t \in \mathbb{Z}_+\), the agent observes its current state as well as the current state of the environment, denoted for simplicity by the vector \(\mathbf{x}_t = [\mathbf{x}^a_{t}, \mathbf{x}^e_{t}]^{\top} \in \mathcal{X}=\mathcal{X}_a \times \mathcal{X}_e\), where \(\mathbf{x}^a_{t} \in \mathcal{X}_a\) represents the agent's state, and \(\mathbf{x}^e_{t} \in \mathcal{X}_e\) represents the environment’s state, i.e., every artifact within the environment outside the agent. The agent applies a control input \(\mathbf{u}_t \sim \mu\left(\cdot \mid \mathbf{x}^a_t\right) \in \mathcal{U}\) according to a (randomized) policy or control law \(\mu: \mathcal{X}_a \to \Delta(\mathcal{U})\), where \(\mathcal{U}\) denotes the control space representing the discrete set of permissible control inputs, \(\Delta(\mathcal{U})\) is the set of distributions over the control space, and \(\sum_{\mathbf{u}_t \in \mathcal{U}} \mu\left(\mathbf{u}_t \mid \mathbf{x}^a_t\right)=1, \ \forall \ \mathbf{x}^a_t \in \mathcal{X}_a\). The agent also incurs a stage cost, \(g_t = \phi(\mathbf{x}^a_t,\mathbf{u}_t)\in \mathbb{R}\), and transitions to a new state according to \(\mathbf{x}^a_{t+1} = f^a(\mathbf{x}^a_t, \mathbf{u}_t)\), where \(f^a\) is the state transition function.
We define the system in which the agent interacts with the environment by a Markov Decision Process (MDP), given by the tuple \(\mathcal{M} = (\mathcal{X}, \mathcal{U}, f^a, f^e, \phi, \gamma)\), where $\gamma$ is a discount factor for expected future costs. 

\begin{example}[\newchangesnorev{Running Example}]\label{ex:reachavoid}
    Consider the navigation task shown in \cref{fig:learn-env}. The robot's state space, $\mathcal{X}_a$, comprises its wheel orientation, angular velocity, linear velocity, and acceleration, goal and obstacle lidar measurements, and Boolean indicators for obstacle collision events (the robot's dynamics are outlined in \cref{sec:exp}). The control set, $\mathcal{U}$, is 2-dimensional and consists of the torques applied to the left and right wheels. The stage cost, $g_t \in \mathbb{R}$, is a function of the Euclidean distance between the robot and the prescribed goal location at each timestep (see \cref{sec:exp}). \newchangesnorev{The robot must also avoid a keep-off region, \(\mathcal{X}^s\) and an obstacle region, \(\mathcal{X}^h\), that both incur a violation cost distinct from \(g_t\).}
\end{example}

\begin{definition}[Safety]\label{def:safety}
\newchangesnorev{
A trajectory \(\{\mathbf{x}^a_t, \mathbf{u}_t\}_{t=0}^T\) is deemed {safe} if it avoids 
 environmental constraint violations, i.e., if
 \begin{equation}
     \mathbf{x}^a_t \notin \mathcal{X}_{\text{unsafe}}=\mathcal{X}^s\cup\mathcal{X}^h\equiv\mathcal{X}_a\setminus\mathcal{X}_{\text{safe}}, \ \forall \ t,
 \end{equation}
 where \(\mathcal{X}_{\text{unsafe}}\) denotes the set of agent states corresponding to collisions and restricted regions, and \(\mathcal{X}_{\text{safe}}\) is its complement.}
\end{definition}

\begin{remark}[Enforcing Safety]
\newchangesnorev{
To avoid the computational challenges associated with deriving an explicit expression for \(\mathcal{X}_{\text{safe}}\), we enforce safety via unique cost functions (specified for each safety constraint) that assign a positive scalar penalty whenever the agent enters an unsafe region:
\begin{equation}\label{eq:safetycost}
C_t:=  \begin{cases}
    c^s_t(\mathbf{x}^a_t) > 0, \ \text{if } \mathbf{x}^a_t \in \mathcal{X}^s, \ 0 \text{ otherwise, }\\
    c^h_t(\mathbf{x}^a_t) > 0, \ \text{if } \mathbf{x}^a_t \in \mathcal{X}^h, \ 0 \text{ otherwise. }
\end{cases}
\end{equation}
To implicitly bias the agent toward safe behavior, we then minimize the cumulative sum of the violation costs, with the costs defined as in \cref{eq:safetycost}. See \cref{eq:models} for the full cost expression.}
\end{remark}
\subsection{Action-Value Iteration (AVI)}\label{ssec:qravi}
In an MDP, the agent seeks a policy \(\mu\) that minimizes the expected cumulative cost. The optimal cost-to-go function satisfies \(J^\star(\mathbf{x}) = \min _{\mathbf{u}_t} Q^{\star}\left(\mathbf{x}_t, \mathbf{u}_t\right),\)
where \( Q^\star(\mathbf{x}_t, \mathbf{u}_t)\) is the optimal action-value function. The action-value function \( Q(\mathbf{x}_t, \mathbf{u}_t) \), which estimates the future cost from a given state-action pair, is updated iteratively using the Bellman recursion
\begin{align}\label{eq:bellupdate}
   Q\left(\mathbf{x}_t, \mathbf{u}_t\right) \leftarrow & 
 \ Q\left(\mathbf{x}_t, \mathbf{u}_t\right) + \alpha^t \big[\phi\left(\mathbf{x}_t, \mathbf{u}_t\right) \nonumber\\
   &+ \gamma \min _{\mathbf{u}_{t+1}} Q\left(\mathbf{x}_{t+1}, \mathbf{u}_{t+1}\right) - Q\left(\mathbf{x}_t, \mathbf{u}_t\right) \big],
\end{align}  
where \(\phi(\mathbf{x}_t, \mathbf{u}_t)\) is the stage cost, \(\alpha^t \in (0,1]\) is the learning rate, and \(\gamma \in (0,1]\) is a discount factor. Since maintaining exact action values for all state-control pairs in large state-control spaces is impractical, we specialize our discussion to approximate value iteration methods in the next section.

\subsection{Approximate AVI via Quantile Regression}
{Approximate} action-value iteration methods replace the tabular representation with a function approximator, such as a \(\theta\)-parameterized neural network, i.e., \(Q(\mathbf{x}_t, \mathbf{u}_t) \approx Q(\mathbf{x}_t, \mathbf{u}_t; {\theta})\), that is typically trained using data sampled uniformly from a sequence of state-control pairs, $\mc{D}$, containing tuples of the form $\{\left(\mathbf{x}_t, \mathbf{u}_t, {g}_t, \mathbf{x}_{t+1}, d_t\right)\}$, where \(d_t\) is a Boolean indicator variable that signifies if a learning episode has ended. The predicted Q-values (denoted by \(Q(\mathbf{x}_t, \mathbf{u}_t; \theta)\)) are updated by minimizing the loss function defined by
\begin{equation}\label{eq:lossavi}
\mathcal{L}(\theta_i) = \mathbb{E} \big[\left(y_i - Q(\mathbf{x}_t, \mathbf{u}_t); \theta_i \right)^2\big], \ i = 1, 2, \ldots, B,
\end{equation}
where $B$ is the size of the batch sampled from \(\mathcal{D}\), \(y_i\) is the \(\ith\) target action-value realized from a target network, \(Q(\mathbf{x}_t, \mathbf{u}_t; \theta_{i})\), with parameters, \(\theta_{i}\) that are updated with a predetermined frequency to converge at a stable or low-variance value: \(y_i := g_t + \gamma \min_{\mathbf{u}_{t+1}} Q(\mathbf{x}_{t+1}, \mathbf{u}_{t+1}; \theta_{i})\).
To account for variability in the Q-value distribution, quantile regression-based approximate AVI ({\qravi}) methods \cite{kuznetsovControllingOverestimationBias2020,dabneyDistributionalReinforcementLearning2017} estimate the target action-value distribution using a set of quantiles \newchangesnorev{(e.g., the lower tail, median, upper tail etc.) instead of predicting just the mean Q-value}. Specifically, in the {\qravi} framework \cite{dabneyDistributionalReinforcementLearning2017}, the loss in \cref{eq:lossavi} is replaced by the quantile regression loss:
\begin{equation}\label{eq:quantloss}
\mathcal{L}_{\text{QR}}(\hat{\theta}_n) := \frac{1}{\numquantiles} \displaystyle\sum_{n=1}^{\numquantiles} \sum_{j=1}^{\numquantiles} \rho^\kappa_{\tau_n}\big[y_j - \hat{\theta}_n(\mathbf{x}_{t}, \mathbf{u}_{t})\big].
\end{equation} 
In \cref{eq:quantloss}, \(y_j = g_t + \gamma \hat{\theta}_j(\mathbf{x}_{t+1}, \mathbf{u}_{t+1})\)
defines the \(\jth\) quantile (of the target action-value distribution) that depends on the system's (agent-environment) transition dynamics and stage cost, with \(j \in \{1, \ldots, N_\tau\}\), \(\numquantiles\) is the prescribed number of quantile fractions, and \(\rho^\kappa_{\tau_n}\) is the asymmetric Huber loss (with parameter \(\kappa\)) corresponding to the uniformly-distributed quantile fraction, \(\tau_n \in (0,1)\) (\(\delta_{w}\) is the Dirac delta function at \(w\in\mathbb{R}\) and \(n = 1, 2, \dots, \numquantiles\)):
\begin{align}
\rho^\kappa_{\tau_n}(m) = |\tau_n - \delta_{\{m < 0\}}|\mathcal{L}_\kappa(m), \ \tau_n = \nicefrac{(n - 0.5)}{\numquantiles}\\
\mathcal{L}_\kappa(m) = \bigg\{
\begin{array}{ll}
\frac{1}{2} m^2, & \text{if } |m| \leq \kappa, \\
\kappa\left(|m| - \frac{1}{2} \kappa\right), & \text{otherwise}.
\end{array}
\end{align}
\begin{remark}[Parameters of \(\mathcal{L}_{\text{QR}}\)]
    The Huber loss function, \(\rho_{\tau_n}^\kappa\), in \cref{eq:quantloss} serves to minimize the action-value distribution approximation error for the \(n\)-th quantile by combining the squared and absolute errors and switching between these errors based on a threshold parameter, \(\kappa\). This switching strategy makes the Huber loss robust to outliers \cite{dabneyDistributionalReinforcementLearning2017}, which is an important requirement for real-world control applications with noisy or uncertain data. The Dirac function adjusts the magnitude of the error based on the sign of the desired cost-to-go (denoted by $\mathrm{m}$) to differentiate between underestimation and overestimation of the action value distribution.
\end{remark}

\subsection{Data-Driven Uncertainty Quantification}\label{ssec:uncqnt}
At each state, \(\mathbf{x}^a_t\), since the agent can choose from multiple possible actions \(\mathbf{u}_t \in \mathcal{U}\) according to a random policy, each random action thus incurs a non-deterministic stage cost given by \(\phi(\mathbf{x}^a_t, \mathbf{u}_t)\), as well as a random safety violation cost not captured by the stage cost function (see \cref{ex:reachavoid} and \cref{def:safety}). \newchangesnorev{We thus model the randomness in the safety violation costs \textit{separately} by assuming \(C_t\) follows an unknown distribution with probability density, \(P_c\), that must be learned from cost samples in the sampled batch, \(\{C_t^{(1)}, C_t^{(2)}, \ldots, C_t^{(|B|)}\}\), for each time step, where \(C_t^{(e)}\) is the violation cost incurred at time step \(t\) in episode \(e\). Due to the violation cost's dependence on the policy, the associated cost distribution at \(\mathbf{x}_t\) is thus conditioned on \(\mu\). We define the notion of a policy-conditioned cost distribution next.}

\begin{definition}[Policy-Conditioned Cost Distribution]\label{def:picostdist}
\noindent Given a policy, \(\mu\), and a cost distribution, \(P_c\), the \textit{policy-conditioned} cost distribution at state \(\mathbf{x}_t\) (hereafter denoted by \(Z_c^\mu(\cdot \mid \mathbf{x}^a_t)\)) is the cost distribution obtained by marginalizing over inputs
\begin{equation}\label{eq:picostdist}
\sum_{\mathbf{u}_t \in \mathcal{U}} \mu\left(\mathbf{u}_t \mid \mathbf{x}^a_t\right) P_c\left(\cdot \mid \mathbf{x}^a_t, \mathbf{u}_t\right),
\end{equation}
where \(P_c\left(\cdot \mid \mathbf{x}^a_t, \mathbf{u}_t\right)\) is a distribution of safety violation costs associated with the state control pair, \(\left(\mathbf{x}^a_t, \mathbf{u}_t\right)\). For instance, if \(C_t\) is the cost incurred from violating a safety constraint, \(P_c\left(C_t\mid \mathbf{x}^a_t, \mathbf{u}_t\right)\) provides the likelihood of incurring a particular cost \(C_t\) given the pair, \(\left(\mathbf{x}^a_t, \mathbf{u}_t\right)\).
\end{definition}

\begin{remark}[Challenges with Computing \(Z_c^\mu(\cdot\mid \mathbf{x}^a_t)\)]\label{rem:compcdist}%
Equation \cref{eq:picostdist} provides a formal framework that fully captures the cost distribution for a given state and under control inputs chosen according to a policy. Unfortunately, due to the unavailability of an exact form of \(P_c(\cdot \mid \mathbf{x}^a_t, \mathbf{u}_t)\), a direct computation of a closed form expression for \(Z_c^\mu(\cdot\mid \mathbf{x}^a_t)\) becomes infeasible. This challenge thus necessitates the need for an approximation of \(Z_c^\mu(\cdot\mid \mathbf{x}^a_t)\), hereafter denoted as \(\hat{Z}^{\mu}_c\), from costs sampled from stored experience during training.
\end{remark}

\subsection{Risk Model \& Risk Measure Computation}\label{ssec:riskmod}
\begin{definition}[Risk Measure]\label{def:cbrisk}\noindent Consider the probability space, \((\Omega, \mathcal{F}, \hat{Z}^{\mu}_c)\), where \(\Omega \supseteq \mathcal{X}\times \mathcal{U}\) is the sample space and \(\mathcal{F}\) is the \(\sigma\)-algebra over \(\Omega\). Suppose \(\mathscr{C}\) denotes the set of all cost random variables defined on \(\Omega\). Omitting the time argument for brevity, we define the risk measure as a function \(\hat{\rho}_\beta: \mathscr{C} \to \mathbb{R}\) that assigns a real number representing the \textit{risk} or variability of \(C\), i.e., \(\hat{\rho}_\beta\) captures the tail behavior of the cost distribution; \(\beta \in (0, 1)\) is the \textit{confidence level} representing the proportion of \(\hat{Z}^{\mu}_c\) that is covered when computing \(\hat{\rho}_\beta\).
\end{definition}
Hereafter, we focus on {\em coherent} risk measures, widely used in risk-sensitive optimization, that satisfy key axioms: sub-additivity, positive homogeneity, translation invariance, monotonicity, and risk-utility duality \cite{artzner_coherent_1999,nass2019entropic, rockafellar2020risk}. Of these, we select the Conditional Value-at-Risk (CVaR), since it can distinguish between tail events \cite{artzner_coherent_1999} beyond \(\beta\). Given the cost distribution \(\hat{Z}^{\mu}_c\) realized from some estimation scheme (see \cref{ssec:kdecostdist}), the expected cost at time step \(t\) is:
\begin{equation}\label{eq:expzdist}
    \mathbb{E}\left[C_t\right]=\frac{1}{B} \sum_{e=1}^B C_t^{(e)}.
\end{equation}
Using \cref{eq:expzdist}, and specializing to the case where \(\hat{\rho}_\beta\) is the \(\beta\)-level CVaR (\(\operatorname{CVaR}_\beta\)), the risk measure is given by
\begin{align}\label{eq:cvarbetaform}
    \mathrm{CVaR}_\beta\left(C_t\right)&=\mathbb{E}\left[C_t \mid C_t \geq \operatorname{VaR}_\beta\left(C_t\right)\right], \text{ where}\\
    \mathrm{VaR}_\beta\left(C_t\right) &= \inf_{\star}\{\star \in \mathbb{R} \mid \hat{Z}^{\mu}_c(C_t\le\star) \geq \beta\}
\end{align}
is the Value-at-Risk or \(\beta^{th}\) percentile of \(\hat{Z}^{\mu}_c\).

\section{Problem Formulation}\label{sec:prbform}
Assume the setting in \cref{sec:RL}. Given an MDP representing the system, we wish to find a policy that minimizes the expected cumulative cost while minimizing the risk of safety constraint violations, that is, the \textit{risk-sensitive} cumulative cost, where the risk is computed from an estimated (constraint-violation) cost distribution. Formally, we consider the following problem:
\begin{problem}\label{problem:main}
Given an MDP (\cref{sec:RL}), a stage cost function, \(\phi\), a cost distribution, \(\hat{Z}^{\mu}_c\), and an initial state, \(\mathbf{x}_0\), find a policy \(\mu^*\) such that:
\begin{align}\label{eq:prbmain}
\mu^* = \arg \min_\mu \mathbb{E}\left[\sum_{t=0}^{T-1} \alpha^t \phi(\mathbf{x}_t, \mathbf{u}_t)\right] +  \lambda[\hat{\rho}_\beta\big(\sum_{t=0}^{T-1}C_t\big)].
\end{align}
\end{problem}%
In \cref{eq:prbmain}, \(\hat{\rho}_\beta\) denotes the \(\beta\)-level CVaR corresponding to \(\hat{Z}^{\mu}_c\), estimated from cost samples collected during each episode, \(\lambda\) is a positive scalar regularization parameter that balances the cumulative stage cost minimization and risk minimization, and \(C_t \sim \hat{Z}^{\mu}_c\) is the violation cost random variable. To solve \cref{problem:main}, we propose a risk-sensitive QR-AVI algorithm, with the risk computation enabled by Kernel Density Estimation (KDE). In \cref{ssec:costdistapx}, we provide a more application-centered argument in support of our choice of KDE, while \cref{sec:approach} (appearing next) describes the specifics of our approach that draw on recent advances in risk-sensitive RL \cite{yingSafeReinforcementLearning2022,mavridis2022risk,noorani2021risk-a,noorani2021risk-b,noorani2022risk,chowRiskConstrainedReinforcementLearning2017}.

\noindent\begin{minipage}{\linewidth}
\vspace{-3pt}
\centering
\includegraphics[width=\linewidth,height=2in]{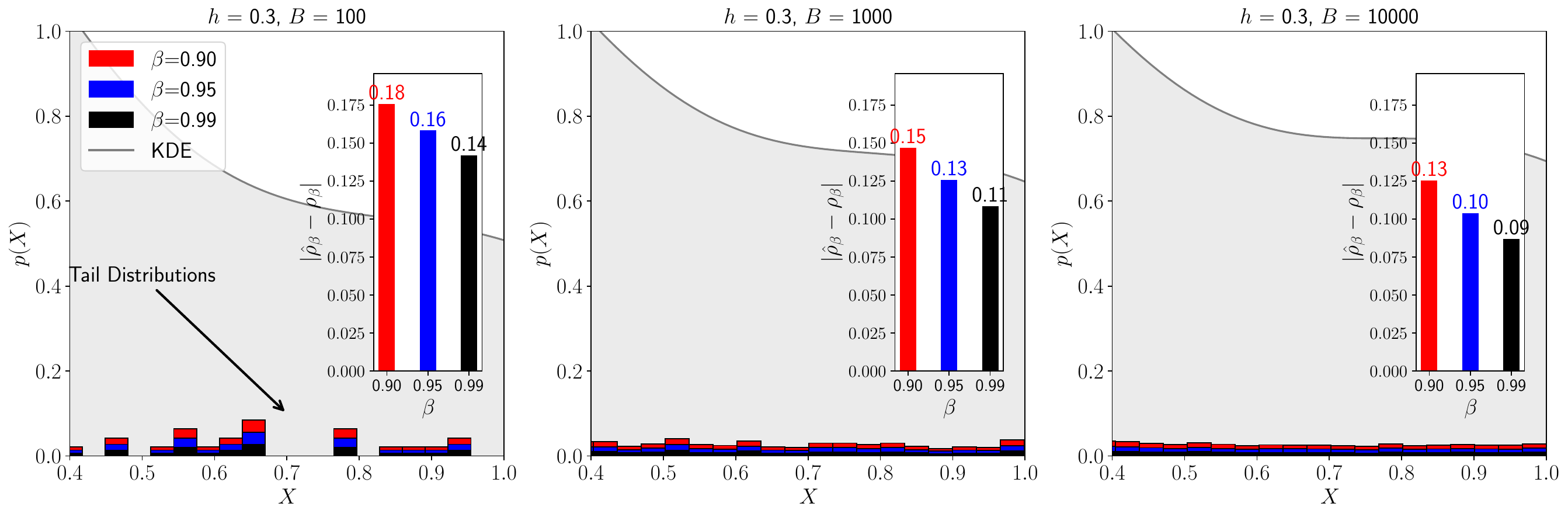}
\captionof{figure}{{Graphing the KDE-estimated probability density (\(p(X)\), shown as a gray line enclosing the gray-filled area) corresponding to samples of an uncertain variable (\(X\)). The inset bar plots show the evolution of the absolute error, \(|\hat{\rho}_{\beta} - {\rho}_{\beta}|\), between the CVaR computed from \(p(X)\) (i.e., \(\hat{\rho}_\beta\)) and the true CVaR (\({\rho}_\beta\)) for an increasing number of KDE samples, \(B = \{100, 1000, 10000\}\) and using a Gaussian kernel with a fixed bandwidth (\(h\)) of 0.3. The underlying data are from a  heavy tailed distribution on \((0, 1]\), and the CVaR estimate moderately improves with the number of cost samples.}}
\label{fig:kdebetaest}
\vspace{-6pt}
\end{minipage}
\section{Approach: Risk-Regularized Quantile-Based AVI with KDE-Based Cost Distribution Estimation}\label{sec:approach}
\subsection[KDE-Based Estimation of the Policy-Constrained Cost Distribution]{KDE-Based Estimation of \(Z_c^\mu\)}\label{ssec:kdecostdist}
Following the discussions in \cref{ssec:uncqnt} and for reasons outlined in \cref{rem:compcdist}, we employ Kernel Density Estimation (KDE) \cite{silverman_density_2017}, a technique that provides a practical method to approximate the policy-conditioned cost distribution (see \cref{def:picostdist}) from finite training samples. By leveraging observed costs during policy execution, KDE constructs an empirical distribution, \(\hat{Z}^{\mu}_c\) \cref{eq:kdecost}, as a proxy for \(Z_c^\mu\), that enables the computation of distributional statistics essential for risk-sensitive decision-making. Denoting the cost samples observed from the training data at time step \(t\) by \(\mathcal{C}_t:=\{{C_t^{(i)}}\}_{i=1}^B\), we can write the likelihood of \(C_t\) as follows (\(k\) is the kernel function, and \(h > 0\) is the bandwidth):%
\begin{equation}\label{eq:kdecost}
\hat{Z}^{\mu}_c(C_t \mid \mathbf{x}_{t})=\frac{1}{Bh} \sum_{i=1}^B k\left(\frac{C_t-C_t^{(i)}}{h}\right), \ C_t^{(i)} \in \mathcal{C}_t.
\end{equation}%

\begin{remark}[KDE Caveats]
    While KDE approximates \(Z_c^\mu\) by smoothing over observed cost values, thus capturing both aleatoric uncertainty (due to the inherent stochasticity in the costs) and epistemic uncertainty (due to finite cost samples) as \newchangesnorev{argued} in \cite{hullermeierAleatoricEpistemicUncertainty2021}, its estimation accuracy (depends on and) may improve albeit marginally with the number of samples \cite{chen_tutorial_2017}. For instance, limited experiments with a heavy-tailed distribution (see \cref{fig:kdebetaest}) reveal that increasing \(B\) hundredfold yields only a 5\% improvement in estimation accuracy.
\end{remark}%

\subsection{Quantile Regression with Risk-Sensitive Loss Functions}\label{ssec:rslossf}
As noted in \cref{sec:intro}, encoding safety with a monolithic cost function has limitations. To better capture safety risks, we propose the following risk-regularized loss function (\(\mathcal{L}\)) for {\qravi} (with \(\lambda \in (0,1)\)):%
\begingroup
\setlength{\abovedisplayskip}{0pt}
\setlength{\belowdisplayskip}{0pt}
\begin{equation}\label{eq:qrriskloss}
\mathcal{L}(\theta)=(1-\lambda)\mathcal{L}_{QR}(\theta)+\lambda \mc{L}_{\rho}(\hat{\rho}_\beta(C), c_{\mathrm{max}}),
\end{equation}
\endgroup
where \(\mc{L}_{QR}(\cdot)\) is the quantile regression loss given by \cref{eq:quantloss}, \(\mc{L}_{\rho}(\cdot)\) is the loss corresponding to the risk (i.e., \(\hat{\rho}_\beta(C)\); see \cref{eq:cvarbetaform}) computed over the cost distribution in \cref{eq:kdecost}, and \(c_{\mathrm{max}}\) is a positive cost threshold. By applying KDE on the observed cost samples \(\mathcal{C}\), we can compute \(\hat{\rho}_\beta(C)\) with respect to \(\hat{Z}^{\mu}_c\). 
\subsection{Risk-Sensitive Approximate Value Iteration with Learned Constraint-Violation Cost Distribution}\label{ssec:riskm}
With the KDE-estimated cost distribution, we construct the risk-sensitive loss function in \cref{eq:qrriskloss} and train a function approximator that approximates the distribution of action-values using quantile regression. In the cost-minimizing context, two typical scenarios result from incorporating risk sensitivity in AVI, depending on \(\beta\)'s value: \(\beta\) values on \([0.9, 1)\) correspond to a \textit{risk-averse} setting, while \(\beta = 0.5\) corresponds to the \textit{risk-neutral} setting, since the CVaR reduces to the expected value of the distribution \cite{rockafellar2020risk}.

Accordingly, to distinguish between these settings in our {\qravi} algorithm, each corresponding to a different instance of \(\mathcal{L}\) (see \cref{tab:loss-functions}), we adopt the respective abbreviations: ${\eqravi}$ and ${\rsqravi}$. Our algorithm (\cref{alg:raqravi}) trains a risk-sensitive function approximator for an action-value distribution by incorporating KDE-based risk estimation into the learning process. The key step that enables the computation of the risk loss is the storing of state-action-\textit{cost} transitions in the replay buffer (\cref{algl:alg:sarcs} of \cref{alg:raqravi}). To update the Q-network, we compute the risk loss using the KDE-estimated distribution of costs sampled from experience.

\begin{minipage}[t]{\textwidth}
\begin{minipage}[t]{0.5\textwidth}
     \removelatexerror
        \begin{algorithm}[H]
            \DontPrintSemicolon
            \LinesNumbered
            \SetNlSty{texttt}{}{}
            \SetAlgoNlRelativeSize{-2.4}
            \SetNlSkip{0.5em}
            \caption{Risk-Sensitive Approximate Value Iteration}
            \label{alg:raqravi}
            
            \KwIn{\(T, \mathcal{D}, B, \gamma, \beta, \lambda, \kappa, \{\epsilon_t\}^{T}_{t=1}, \eta, {\numquantiles}, c_{\mathrm{max}}\)}
            \KwOut{\(Q_\theta\)}
            
            Initialize \(Q_{\theta}, Q_{\theta'}\), \(\mathcal{D}\), \(\{\tau_n\}_{n=1}^{{\numquantiles}},\,d_t={\texttt{\small false}}\)\;
            Observe state $\mathbf{x}_{0}$\;

            \For{$t = 0$ \textbf{to} $T$}{
                
                \While(\tcp*[h]{\(\mathbf{x}_{t+1}\) is not terminal}){\(d_t \neq \text{\normalfont \small \texttt{true}}\)}{
                    \(\mathbf{u}_t \sim \epsilon_t\)-greedy\;
                    
                    Execute \(\mathbf{u}_{t}\), compute \(g_t\), observe \(\mc{C}_t, \mathbf{x}_{t+1}, d_t\)\;
                    
                    \ShowLn\label{algl:alg:sarcs}Store \((\mathbf{x}_t, \mathbf{u}_t, g_t, \mc{C}_t, \mathbf{x}_{t+1}, d_t)\) in \(\mathcal{D}\)\;\nllabel{alg:line:sarcsa}
                    \If{\(|\mathcal{D}| \ge B\)}{
                        Sample batch \(B\) from \(\mathcal{D}\)\;\label{alg:line:batch1}
                        
                        \For{$(\mathbf{x}_{t,j}, \mathbf{u}_{t,j}, {g}_{t,j}, \mc{C}_{t,j}, \mathbf{x}_{t+1,j}, d_{j,t})$ in $B$}{
                            \(Q_{target} = g_{t,j} + \gamma (1 - d_{j,t}) \min_{\mathbf{u}_{t+1}} \mathbb{E}[\hat{\theta}(\mathbf{x}_{t+1,j}, \mathbf{u}_{t+1}, \tau)]\)\;\label{alg:line:batch2}}
                        {\textbf{UpdateNetwork}}(\(Q_{target}, \mc{C}_{t,j}, \texttt{\small args*}\))\;
                        \(\theta' \gets \eta \theta + (1 - \eta) \theta'\)\ \tcp*[h]{Update target}\;
                       }
                   }
               }
        \end{algorithm}
\end{minipage}
\hspace{0.2cm}
\begin{minipage}[t]{0.485\textwidth}
    \centering
    \captionof{table}{Training Loss Functions ($C \sim \hat{Z}^{\mu}_c$). Contributions in {\color{blue}blue}.}
    \label{tab:loss-functions}
    \resizebox{\textwidth}{!}{%
    \begin{tabular}{L{1.7cm}L{6cm}} 
    \toprule 
    {\textbf{Algorithm}} & {\textbf{Loss Function}} ($\mathcal{L}$) \\
    \midrule 
    {\avi} & $(y - Q(\mathbf{x}, \mathbf{u}))^2$\\{\qravi} & $\newchangesnorev{\mathcal{L}_{\text{QR}}(\hat{\theta}_n) := \frac{1}{\numquantiles} \displaystyle\sum_{n=1}^{\numquantiles} \sum_{j=1}^{\numquantiles} \rho^\kappa_{\tau_n}\big[y_j - \hat{\theta}_n\big]}$\\
    {${\eqravi}$} & $\textcolor{blue}{\lambda^\prime}\mathcal{L}_{\text{QR}} + \textcolor{blue}{\lambda [(\mathbb{E}[C] - c_{\mathrm{max}})^2]_+}$ (see \cref{eq:expzdist}) \\ 
    {${\rsqravi}$} & $\textcolor{blue}{\lambda^\prime}\mathcal{L}_{\text{QR}} + \textcolor{blue}{\lambda [({\hat{\rho}}_{\beta}(C) - c_{\mathrm{max}})^2]_+}$ (see \cref{eq:cvarbetaform})\\ 
    \bottomrule
    \end{tabular}%
    }
\end{minipage}
\end{minipage}

\section{Theoretical Guarantees for Risk-Sensitive {\qravi}: Contraction, Fixed-Point Existence, \& Finite-Time Convergence}\label{sec:thprop}%
\begin{definition}[Risk-Sensitive Distributional Bellman Operator \cite{lim_distributional_2022}]\label{def:rsbellop}%
Let \( Z \in \mathcal{Z} \) denote the expected cumulative cost distribution, with \( \mathcal{Z} \) given as \( \mathcal{Z} = \{ Z \mid \mathbb{E}[Z(\mathbf{x}, \mathbf{u})] < \infty, \forall\ 
 \mathbf{x}, \mathbf{u}\}\). Let \(\mathcal{Z}_c:=\{ Z^\mu_c \mid \mathbb{E}[Z^\mu_c(\mathbf{x}, \mathbf{u})] < \infty, \forall \ \mathbf{x}, \mathbf{u}\}\). Suppose \(\hat{Z}^{\mu}_c \in \mathcal{Z}_c\) and \(\hat{\rho}_\beta(\hat{Z}^{\mu}_c) \in \mathfrak{C} := \{C|C(\mathbf{x}^a) < \infty, \forall \ \mathbf{x}^a\}\), with \(\hat{\rho}_\beta[\hat{Z}^{\mu}_c](\mathbf{x}^a):=\hat{\rho}_\beta(\hat{Z}^{\mu}_c[\cdot|\mathbf{x}^a])\). The risk-sensitive distributional Bellman operator, \(\mathcal{T}^{\beta}\), corresponding to a \( \beta\)-level coherent risk measure, \(\hat{\rho}_\beta \) (see \cref{def:cbrisk}), is
\begin{equation}\label{eq:rsbellop}
   \mathcal{T}^\beta{Z}(\mathbf{x}, \mathbf{u}): \stackrel{D}{=} {g}_t - \gamma \hat{\rho}_\beta[Z^{\mu}_c(\mathbf{x}_{t+1}, \mathbf{u}_{t+1})].
\end{equation}
\(\overset{D}{:=}\) denotes distributional equivalence, and \(\mathbf{u}_{t+1}\sim \mu(\cdot|\mathbf{x}_{t+1})\).
\end{definition}

\begin{theorem}[Contraction of the Risk-Sensitive Bellman Operator with Cost-Based Risk Regularization]\label{prp:ctrbell}%
Assume the setting in \cref{def:rsbellop}. Let \(Z_1\) and \(Z_2\) be two cumulative cost distributions, and suppose that \(\hat{\rho}_\beta\) is \textit{non-expansive}, i.e., that it satisfies the following relationship for two random costs, \(C_1 \sim \hat{Z}_{1,c}^{\mu}
 \in \mathcal{Z}_c \) and \(C_2 \sim \hat{Z}_{2,c}^{\mu} \in \mathcal{Z}_c\) (with \(\hat{Z}_{1,c}^{\mu} \neq \hat{Z}_{2,c}^{\mu}\): \(
|\hat{\rho}_\beta(C_1)-\hat{\rho}_\beta(C_2)| \leq\|C_1-C_2\|.\)
Let \(\hat{Z}^{\mu}_c\) be related to ${Z}$ by $\hat{Z}^{\mu}_c(\mathbf{u}|\mathbf{x}^a) = \Psi({Z(\mathbf{x}, \mathbf{u})})$, where \(\Psi: \mathbb{R}\rightarrow \mathbb{R}_{+}\) is a coherent, Lipschitz continuous, and bounded function with Lipschitz constant, \(L \in (0,1)\). Then $\mathcal{T}^{\beta}$ is a $\gamma$-contraction in the Wasserstein-1 (\(W_1\)) metric (see \cite{dabneyDistributionalReinforcementLearning2017}), i.e.,
\(W_1(\mathcal{T}^{\beta} {Z}_1, \mathcal{T}^{\beta} {Z}_2) \leq \gamma W_1({Z}_1, {Z}_2)\).
\end{theorem}
\begin{proof}
Using \cref{eq:rsbellop}, we can write:
\begin{equation} \label{eq:bellop}
\mathcal{T}^{\beta} Z(x, u) \overset{D}{=} \phi(\mathbf{x}, \mathbf{u}) - \gamma \hat{\rho}_\beta(Z_c^\mu(\mathbf{x}_{t+1}, \mathbf{u}_{t+1})),
\end{equation}
so that by applying \(\mathcal{T}^{\beta} \) to \(Z_1\) and \(Z_2\), we obtain:
\begin{equation} \label{eq:topapp}
\begin{array}{rl}
    \mathcal{T}^{\beta} Z_1 &= \phi(\mathbf{x}, \mathbf{u}) - \gamma \hat{\rho}_\beta(\hat{Z}_{1,c}^\mu), \\
    \mathcal{T}^{\beta} Z_2 &= \phi(\mathbf{x}, \mathbf{u}) - \gamma \hat{\rho}_\beta(\hat{Z}_{2,c}^\mu).
\end{array}
\end{equation}
Taking \(W_1\), we obtain \(W_1(\mathcal{T}^{\beta} Z_1, \mathcal{T}^{\beta} Z_2)\) as:
\begin{equation} \label{eq:wass1}
 W_1(\phi(\mathbf{x}, \mathbf{u}) - \gamma \hat{\rho}_\beta(\hat{Z}_{1,c}^\mu), \phi(\mathbf{x}, \mathbf{u}) - \gamma \hat{\rho}_\beta(\hat{Z}_{2,c}^\mu)).
\end{equation}
Since \(\phi(\mathbf{x}, \mathbf{u}) \) is deterministic, it cancels out, yielding:
\begin{equation} \label{eq:wass2}
W_1(-\gamma \hat{\rho}_\beta(\hat{Z}_{1,c}^\mu), -\gamma \hat{\rho}_\beta(\hat{Z}_{2,c}^\mu)).
\end{equation}
By the non-expansiveness of \(\hat{\rho}_\beta\):
\begin{equation} \label{eq:nonexp}
W_1(\hat{\rho}_\beta(\hat{Z}_{1,c}^\mu), \hat{\rho}_\beta(\hat{Z}_{2,c}^\mu)) \leq W_1(\hat{Z}_{1,c}^\mu, \hat{Z}_{2,c}^\mu).
\end{equation}
Multiplying the L.H.S of \cref{eq:nonexp} by \(\gamma\), it follows that:
\begin{align} 
W_1&(-\gamma \hat{\rho}_\beta(\hat{Z}_{1,c}^\mu), -\gamma \hat{\rho}_\beta(\hat{Z}_{2,c}^\mu))\nonumber\\\label{eq:wass3}
&= \gamma W_1(\hat{\rho}_\beta(\hat{Z}_{1,c}^\mu), \hat{\rho}_\beta(\hat{Z}_{2,c}^\mu)) \leq \gamma W_1(\hat{Z}_{1,c}^\mu, \hat{Z}_{2,c}^\mu).
\end{align}
By invoking the \newchangesnorev{bounded-expectation assumption of the distributions in \(\mc{Z}_c\) \newchangesnorev{(since \(\hat{Z}_{1,c}^\mu, \hat{Z}_{2,c}^\mu \in \mc{Z}_c\))}}, we get:
\begin{equation} \label{eq:viocost}
W_1(\hat{Z}_{1,c}^\mu, \hat{Z}_{2,c}^\mu) \leq W_1(Z_1, Z_2).
\end{equation}
By the assumptions on \(\Psi\), there exists an \(0 < L < 1\) s.t.
\begin{equation}\label{eq:wassfin}
W_1(\Psi(Z_1), \Psi(Z_2)) \leq L \cdot W_1(Z_1, Z_2),
\end{equation}
\newchangesnorev{i.e., \(\Psi\) is an \(L\)-contraction with respect to \(||~||\).} \newchangesnorev{Since \({L < 1}\) and \(||\mathcal{T}^{\beta} Z_1 - \mathcal{T}^{\beta} Z_2|| \le ||Z_1 - Z_2||\)}, applying \cref{eq:rsbellop} yields:
\begin{equation}
W_1(\mathcal{T}^{\beta} Z_1, \mathcal{T}^{\beta} Z_2) \leq \gamma W_1(Z_1, Z_2).
\end{equation}
Hence, \( \mathcal{T}^{\beta}\) is a \(\gamma\)-contraction.
\end{proof}

\begin{corollary}[Existence of a Unique Cumulative Cost Distribution]\label{cor:unqretdist}%
From \cref{prp:ctrbell}, by Banach's fixed-point theorem \cite{bharucha-reidFixedPointTheorems1976}, there exists a unique fixed point \( Z^* \) such that:
\begin{equation} \label{eq:fixed_point}
Z^* = \mathcal{T}^{\beta} Z^*.
\end{equation}
\end{corollary}

\begin{remark}[Convergence of Risk-Sensitive {\qravi}]
Suppose $\mathcal{X}$ and $\mathcal{U}$ are continuous and bounded, with \(\mathcal{L}\) and \(\mathcal{L}_\rho\) convex, and assume that the agent explores the environment sufficiently often. Then, under a diminishing learning rate schedule such as 
\(\alpha^t = k_\alpha / (t+1)\), where \(k_\alpha\) controls the decay rate, the risk-sensitive {\qravi} algorithm will converge to a Pareto-optimal action-value in finite time.
This result follows from the convexity of the combined quantile and risk-sensitive losses, which guarantees that any local Pareto optimum is also a global one \cite{van2014multi}, together with standard stochastic approximation results that ensure convergence under diminishing step sizes and sufficient exploration \cite{bertsekas_model_2024}.
\end{remark}

\section{Experiments}
\label{sec:exp}
In this section, we provide a comprehensive overview of our experimental study, detailing our simulation setup, followed by our training and evaluation procedures, and concluding with comparisons with a risk-neutral variant. 
\subsection{Safety-Gymnasium Learning Environment}
\label{ssec:learnenv}
In this experiment, we train a differentially-driven mobile robot (depicted as a red wheeled car in \cref{fig:learn-env}) to navigate to a randomized goal location (\(\mathbf{x}^g \in \mathbb{R}_2\)) while minimizing the cost of traversing restricted areas and avoiding obstacles. The robot is equipped with a lidar sensor to measure distances to various environmental features such as obstacles and the goal. The state space, control space, and the observations that the agent makes, as well as the goal of the experiment, are outlined below:
\begin{subequations}\label{eq:models}
\begin{alignat}{2}
\mathcal{X} &= 
    \left\{\mathbf{x}_t = (\mathbf{x}^a_t,\, \mathbf{x}^s,\, \mathbf{x}^h,\, \mathbf{x}^g)
    \;\middle|\; t = 1,2,\ldots,T \right\}
    &\qquad&\textcolor{blue}{(\text{state space})}
    \label{eq:xmodel}
\\[4pt]
\mathcal{U} &=
    \left\{ \mathbf{u}_t = [v^L_t,\, v^R_t]
    \;\middle|\; t = 1,2,\ldots,T \right\}
    &\qquad&\textcolor{blue}{(\text{control space})}
\\[6pt]
C_t &=
\begin{cases}
c^s_t(\mathbf{x}^a_t) = c_s \sim \mathrm{U}_{[0,1]}\,\mathbb{I}_{\mathcal{X}^s}(\mathbf{x}^a_t),\\[2pt]
c^h_t(\mathbf{x}^a_t) = c_h \sim \mathrm{U}_{[0,1]}\,\mathbb{I}_{\mathcal{X}^h}(\mathbf{x}^a_t),\quad c_h > c_s,
\end{cases}
    &\qquad&\textcolor{blue}{(\text{safety-violation costs})}
    \label{eq:sfcosts}
\\[10pt]
f^a(\mathbf{x}^a_t, \mathbf{u}_t) &=
\begin{cases}
x_{t+1}      = x_t + \tfrac{r}{2}(v_t^R + v_t^L)\cos\theta_t,\\
y_{t+1}      = y_t + \tfrac{r}{2}(v_t^R + v_t^L)\sin\theta_t,\\
\theta_{t+1} = \theta_t + \tfrac{r}{2d_w}(v_t^R - v_t^L),
\end{cases}
    &\qquad&\textcolor{blue}{(\text{transition dynamics})}
\end{alignat}
\end{subequations}

In \cref{eq:models}, \(\mathbf{x}^a_t = [x_t, y_t, \theta_t]^\top\) defines the robot's state vector comprising the position of its \(t\)-step center of curvature and orientation (\(\theta_t\)), \newchangesnorev{and \(\mathbf{x}^s \text{ and } \mathbf{x}^h\) respectively denote the position of the restricted region markers and obstacles}. \(v^L_t\) and \(v^R_t\) are the robot's left and right wheel (linear) velocities, \(r\) is its wheel radius, \(d_w\) is the half-distance between the robot's two wheels, and \(\mathbb{I}_{\mc{X}_\star}\) is an indicator function. The environment contains 10 restricted region markers, 10 obstacles scattered around the map, and a single goal location. The restricted regions, goal, and obstacle locations are randomized according to a uniform distribution, \(\mathrm{U}_{(a, b)}\), {at the beginning of each episode (with \(a, b \in [-1, 1]\)) and remain static (and hence do not depend on \(t\) in \cref{eq:xmodel}) till the episode ends}, i.e., after a prescribed number of time steps, \(T\). 

The observation space consists of 72 dimensions, with the last 48 representing 16 lidar measurements of the distance to the goal, restricted regions, and obstacles within a 3-meter range. Restricted areas are traversable and less severe constraints to avoid, while obstacles (cubes) are hard constraints that must be strictly avoided, incurring higher scalar costs (\(c_h > c_s\); see \cref{eq:sfcosts} for the definition of each cost variable). The agent must minimize the cumulative cost of violating these constraints over a training episode, as in the following equation: \(\sum_{t=0}^{T-1}{(c_h + c_s)}.\)
As such, the optimal policy may not always be able to ensure total avoidance of the obstacles or restricted areas, and must rather minimize this cost. {There are thus three levels of uncertainty arising from randomness or noise in: \emph{lidar measurements}, \emph{robot drift}, and \emph{randomization of environment entities}.} An episode ends after the prescribed maximum time horizon, $T$, setting a terminal state indicator, $d$, to 1. During an episode, the robot may reach as many goals as possible.
\smallskip
\subsubsection{Training and Evaluation}
The agent learns to minimize the stage cost, defined by the change in Euclidean distance to the static goal, \(\mathbf{x}^g\), between consecutive time steps, as in: 
\[g_t = -\gamma(||\mathbf{x}^a_{t} - \mathbf{x}^g||) \,-\, ||\mathbf{x}^a_{t-1} - \mathbf{x}^g|| - \cgoal,\]
where \(\cgoal\) is a positive scalar. To encourage exploration, we applied an epsilon-greedy linear exploration schedule, i.e., {{\(\max\{t(\epsilon_T - \epsilon_0)\Delta t^{-1}, \epsilon_T\}\)}}. To train all permutations of the {\qravi} algorithm (with \(N_\tau = 32\) quantile fractions), we employed a feedforward neural network with three fully-connected (dense) layers containing two hidden layers (of dimension 120 and 84, respectively) with ReLU activation functions, and an output layer that produces values for each (control) input-quantile pair at each time step.

For the KDE cost distribution estimation, we used a Gaussian kernel function with a bandwidth determined by \textit{Scott's rule} \cite{scott_biased_1987}, i.e., \(h = B^{0.2}\), and trained ${\rsqravi}$ for $\beta = 0.9$ and $0.95$. On \cref{fig:tr-curves}, we plot training curves for the expected cumulative cost \cref{fig:tr-curves-a}, risk \cref{fig:tr-curves-b}, and quantile loss \cref{fig:tr-curves-c} for all {\qravi} variants. We evaluated the models using five random seeds (0, 5, 10, 15, 20), with 20 episodes per seed, resulting in a total of 100 test episodes. Each seed ensures reproducibility by initializing the random number generator. Next, we computed the average action value, constraint violation cost, number of successful episodes, and goal success rate over all seeds to compare the performance of the algorithms. Here, the success rate represents the number of times the agent reaches the goal, averaged over all 100 test episodes. \cref{tab:learnparams} summarizes our RL hyperparameters.
\begin{figure*}[htb]
    \centering
    \hspace{-10pt}
    \subfloat[Cost-to-go per time step.]{%
    \label{fig:tr-curves-a}%
    \includegraphics[width=.35\columnwidth,trim={0 0 0cm 0cm},clip ]{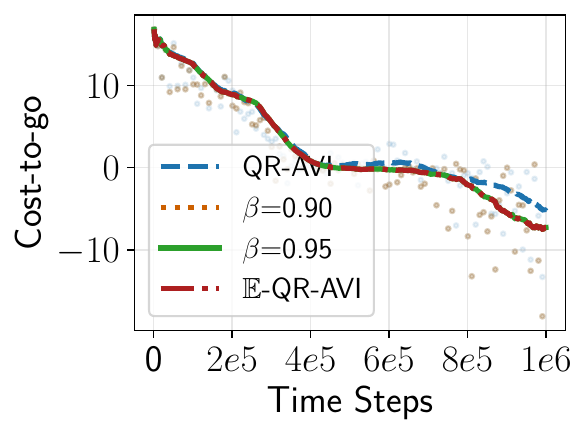}}
    \subfloat[Risk loss ($\mathcal{L}_\rho$).]{%
    \label{fig:tr-curves-b}%
    \includegraphics[width=.35\columnwidth,trim={0 0 0cm 0cm},clip]{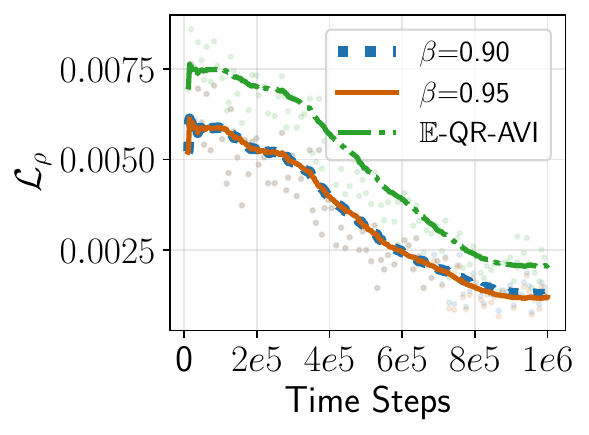}}
    \subfloat[Quantile loss ($\mathcal{L}_{QR}$).]{%
    \label{fig:tr-curves-c}%
    \includegraphics[width=.35\columnwidth,trim={0 0 0cm 0cm},clip]{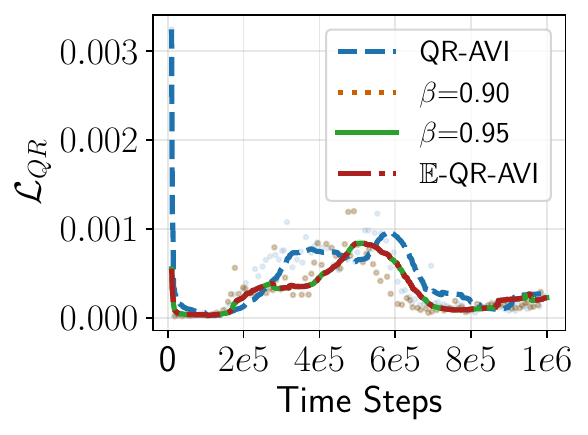}}
    \caption{Evolution of the (training) expected cumulative cost, risk loss, and quantile loss for the reach-avoid navigation task.}
    \label{fig:tr-curves}
\end{figure*}
\begin{table*}[htb]
    \centering
    \small
    \caption{Training and Evaluation Hyperparameters}
    \label{tab:learnparams}
    \begin{NiceTabular}{llllll}
        \thickhline
        \textbf{Param.} & \textbf{Value} & \textbf{Param.} & \textbf{Value} & \textbf{Param.} & \textbf{Value} \\
        \midrule
        $B$, $\alpha$ & $128$, $2.5e$-$4$ & $c_{\mathrm{max}}$, $\beta$ & $0.1$, $\{.9, .95\}$ & $\gamma$, $\numquantiles$, $T$ & $0.99$, $32$, $1000$ \\
        $\epsilon_0$, $\epsilon_T$, $\kappa$ & $1.0$, $0.05$, $1.0$ & $|\mathcal{D}|$ & $0.5e6$ & Upd. Freq. & $10$ (train), $500$ (target) \\
        \thickhline
    \end{NiceTabular}
\end{table*}

\subsection{Juxtaposition with the Risk-Neutral Algorithm}\label{ssec:benchm}
We compare the risk-sensitive {\qravi} against the nominal and expected-value configurations across evaluation metrics. \cref{tab:benchm} juxtaposes our risk-sensitive adaptation with the baselines on the training and evaluation metrics. From the plots, we notice a marked improvement in the average expected cumulative cost for the risk-sensitive {\qravi} than the baseline, with a mostly consistent cost evolution. The agent demonstrates a success rate (that is approximately \textbf{37.66\%} and \textbf{87.45\%} higher per episode) than the baselines (nominal and risk-sensitive, respectively). This is reflected in both the increased success rate and the higher total number of successes across all scenarios and episodes.

\begin{table*}[ht]
\small
\centering
\caption{Comparison of success metrics and loss values across experiments. Mean and standard deviation of the evaluation cost were computed over 100 test episodes. Results using our method are highlighted in \colorbox{gray!20}{gray}.}
\label{tab:benchm}
\resizebox{\textwidth}{!}{%
\begin{tabular}{p{2.45cm}C{1.5cm}C{2cm}C{2cm}C{2cm}C{2cm}C{2.8cm}}
\toprule
\textbf{Alg.} & \adjustbox{stack=cc}{\textbf{Avg. Eval.} \\ \textbf{Cost-to-go}} & \adjustbox{stack=cc}{\textbf{Constraint-} \\ \textbf{Violation Cost}} & \adjustbox{stack=cc}{\textbf{Quantile} \\ \textbf{Loss (Avg.)}} & \adjustbox{stack=cc}{\textbf{Total} \\ \textbf{Loss}} &  \adjustbox{stack=cc}{\textbf{Total}\\\textbf{Goals Reached}} & \adjustbox{stack=cc}{\textbf{Normalized\\\textbf{Success Rate (\%)}}} \\
\midrule
{{\qravi}} & $-0.01$ & $0.05 \pm \mathbf{0.01}$ & $0.0004$ & $0.0004$ & $239$ & $50\%$ \\
{${\eqravi}$} & $-0.01$ & $0.06 \pm \mathbf{0.01}$ & $0.0003$ & $0.0032$ & $329$ & $75\%$ \\
\rowcolor{gray!20}{${\rsqravi}^{{\textcolor{blue}{\bm{\beta=0.9}}}}$} & $-0.01$ & $0.05 \pm \mathbf{0.01}$ & $0.0003$ & $0.0023$ & $448$ & $100\%$ \\
\rowcolor{gray!20}{${\rsqravi}^{{\textcolor{blue}{\bm{\beta=0.95}}}}$} & $-0.01$ & $0.05 \pm \mathbf{0.01}$ & $0.0003$ & $0.0023$ & $448$ & $100\%$\\
\bottomrule
\end{tabular}%
}
\end{table*}
\begin{figure}[htb]
    \centering
        \subfloat[Pareto front (dashed line) corresponding to data points of the quantile loss (\(\mathcal{L}_{QR}\), y-axis) and risk loss (\(\mathcal{L}_{\rho}\), x-axis), for the risk-sensitive ({\(\beta=0.95\) \tikzcircle[black, fill=parblue]{2pt}}; {\(\beta=0.9\) \tikzsquare{parorange}}) and risk-neutral ({\tikztriangle{pargreen}}) cases.]{%
        \label{fig:paretocurve}        \includegraphics[width=.45\columnwidth]{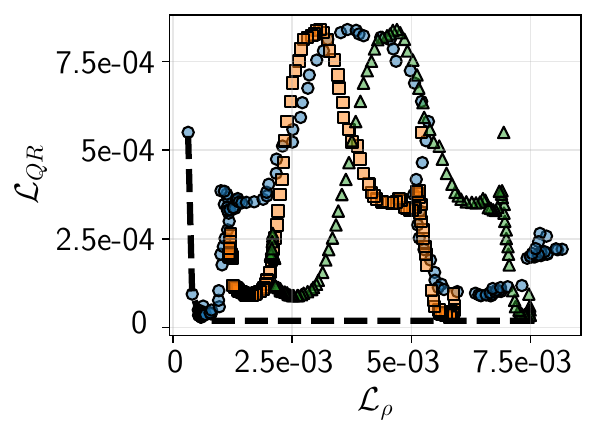}}
        \subfloat[Comparing the \emph{tails} of the KDE-estimated cost distribution and a normal distribution with mean \(0.06\) and standard deviation, \(0.1\).]{%
        \label{fig:distcomp}%
        \includegraphics[width=.45\columnwidth]{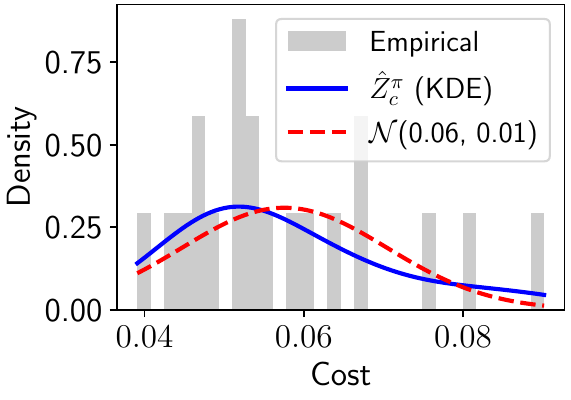}}
    \caption{Pareto front and cost distribution approximation.}
\end{figure}%

\begin{table}[htb]
    \centering
    \caption{$\beta$ vs. the tail probability difference (\(\Delta_{\text {tail}}\)) between \(\hat{Z}^{\mu}_c\) (KDE-based) and $\mathcal{N}(0.06, 0.01)$.}
    \label{tab:kdevsnormptail}%
    \renewcommand{\arraystretch}{0.9}  
    \begin{tabular}{C{1.2cm}C{1.8cm}}  
    \toprule 
    \(\bm{\beta}\) & \(\bm{\Delta_{\text{tail}}}\) \\
    \midrule
    0.90 & 0.042 \\
    0.95 & 0.041 \\
    0.99 & 0.027 \\
    \bottomrule
    \end{tabular}
\end{table}

\subsection{Discussions}
\label{ssec:dis}
Here, we expound on the findings presented in the foregoing subsections and highlight a few implementation details.
\smallskip
\subsubsection{\normalfont\textbf{Training and Evaluation Metrics}}\label{ssec:trevalmet} Comparing the average quantile losses (\cref{tab:benchm}, column 4) reveals that all algorithms achieve near-identical performance, with average losses ranging between \(3e{\text{-}4}\) and \(4e{\text{-}4}\). The nominal {\qravi} exhibits the highest quantile loss, while the risk-sensitive (\(\beta = 0.9\) and \(\beta = 0.95\)) and expected-value variants achieve slightly lower values of \(3e{\text{-}4}\)). Comparisons between ${\rsqravi}$ and ${\eqravi}$ show no significant difference in quantile loss, indicating comparable regression accuracy. Both variants of ${\rsqravi}$ also yield similar returns and costs to the baselines but with a substantially higher success rate.

\smallskip
\subsubsection{\normalfont\textbf{Risk-Performance Trade-offs}}\label{ssec:riskperftrdf}
Given that we trained the {{\qravi}} model on a single neural network, a natural question about the performance-risk trade-off arises. Here, performance is quantified by the minimum quantile loss. The Pareto curve, presented in \cref{fig:paretocurve}, depicts the trade-off between the quantile (\(\mathcal{L}_{QR}\)) and risk losses (\(\mathcal{L}_{\rho}\)) for the expectation and risk-sensitive configurations of {\qravi}. The plot shows that \({\eqravi}\) achieves the lowest quantile loss, indicating better accuracy in predicting cumulative cost, but at the expense of a significantly higher risk loss, highlighting its limitations in risk-sensitive settings. In contrast, the risk-sensitive configurations with \(\beta=0.9\) and \(\beta=0.95\) prioritize safety by achieving lower risk loss values while incurring moderate increases in quantile loss. For safety-critical tasks, the setting with \(\beta=0.95\) is preferable due to its low risk loss, despite a marginally higher quantile loss.%

\smallskip
\subsubsection{\normalfont\textbf{Determining a Fitting Cost Distribution Approximation}}\label{ssec:costdistapx}
Due to environmental randomization, batch costs from the replay buffer may produce empirical distributions with light {\em right} tails (\cref{fig:distcomp}) and negligible probability mass for extreme costs, leading to non-informative upper quantiles that coincide with the mean (\cref{tab:kdevsnormptail}). We thus favored KDE over simpler methods like softmax or the normal distribution for a more accurate cost distribution approximation.

\section{Conclusions}\label{sec:conc}
We introduced a risk-sensitive quantile-based action-value iteration algorithm that balances safety and performance by augmenting the quantile loss with a risk term encoding safety constraints. Our results show that risk sensitivity preserves quantile regression accuracy and ensures consistent performance with tunable risk aversion. The method guarantees convergence to a unique risk-sensitive cost distribution, providing a theoretical foundation. The risk measure is compatible with any off-policy RL model and can be integrated into gradient ascent. Future work will explore dynamic risk parameter adjustments for improved trade-offs in varying conditions.


%
\acknowledgments{This work was partially supported by the Army Research Laboratory Cooperative Agreement No. W911NF-23-2-0040, a Northrop Grumman Corporation grant, and by the Lockheed Martin Chair in Systems Engineering.}


\bibliography{arxiv}

\end{document}